\documentclass{article}

\usepackage{microtype}
\usepackage{graphicx}
\usepackage{subfigure}
\usepackage{booktabs} % for professional tables
\usepackage{tikz}
\usepackage{wrapfig}

\usepackage{booktabs}    % For professional-looking tables
\usepackage{multirow}    % For multi-row cells in tables
\usepackage{array}       % For advanced table column formatting
\usepackage{siunitx}     % For aligning numbers in tables

\usepackage{hyperref}

\usepackage[accepted]{icml2025}

\usepackage{amsmath}
\usepackage{amssymb}
\usepackage{mathtools}
\usepackage{amsthm}

\usetikzlibrary{shapes.geometric, shapes,decorations,arrows,calc,arrows.meta,fit,positioning,
backgrounds}
\tikzset{
    -Latex,auto,node distance =1 cm and 1 cm,semithick,
    state/.style ={ellipse, draw, minimum width = 0.7 cm},
    point/.style = {circle, draw, inner sep=0.04cm,fill,node contents={}},
    bidirected/.style={Latex-Latex,dashed},
    el/.style = {inner sep=2pt, align=left, sloped}
}

\usepackage[capitalize,noabbrev]{cleveref}

\theoremstyle{plain}
\newtheorem{theorem}{Theorem}[section]

\theoremstyle{definition}

\theoremstyle{remark}

\usepackage[textsize=tiny]{todonotes}

\icmltitlerunning{Doubly Robust Monte Carlo Tree Search}

\begin{document}

\twocolumn[
\icmltitle{Doubly Robust Monte Carlo Tree Search}



\icmlsetsymbol{equal}{*}

\begin{icmlauthorlist}
\icmlauthor{Manqing Liu}{yyy,sch}
\icmlauthor{Andrew L. Beam}{yyy,comp}
\end{icmlauthorlist}

\icmlaffiliation{yyy}{Department of Epidemiology, CAUSALab, Harvard University, Boston, USA}
\icmlaffiliation{sch}{School of Engineering and Applied Sciences, Harvard University, Cambridge, USA}
\icmlaffiliation{comp}{Lila Sciences, Cambridge, USA}
%\icmlaffiliation{sch}{School of ZZZ, Institute of WWW, Location, Country}

\icmlcorrespondingauthor{Andrew Beam}{andrew\_beam@hms.harvard.edu}
\icmlkeywords{Monte Carlo Tree Search; Doubly Robust Estimation; Reinforcement Learning}

\vskip 0.3in
]



\printAffiliationsAndNotice{}  % leave blank if no need to mention equal contribution

\begin{abstract}
We present Doubly Robust Monte Carlo Tree Search (DR-MCTS), a novel algorithm that integrates Doubly Robust (DR) off-policy estimation into Monte Carlo Tree Search (MCTS) to enhance sample efficiency and decision quality in complex environments. Our approach introduces a hybrid estimator that combines MCTS rollouts with DR estimation, offering theoretical guarantees of unbiasedness and variance reduction under specified conditions. Empirical evaluations in Tic-Tac-Toe and the partially observable VirtualHome environment demonstrate DR-MCTS's superior performance over standard MCTS. In Tic-Tac-Toe, DR-MCTS achieves an 88\% win rate compared to a 10\% win rate for standard MCTS. In compound VirtualHome tasks, DR-MCTS attains a 20.7\% success rate versus 10.3\% for standard MCTS. Our scaling analysis reveals that DR-MCTS exhibits better sample efficiency, notably outperforming standard MCTS with larger language models while using a smaller model. These results underscore DR-MCTS's potential for efficient decision-making in complex, real-world scenarios where sample efficiency is paramount.
\end{abstract}
\section{Introduction}

Decision-making in complex, partially observable environments remains a fundamental challenge in artificial intelligence. Monte Carlo Tree Search (MCTS) has emerged as a powerful approach for addressing this challenge, demonstrating remarkable success in domains ranging from game playing to robotics \cite{browne2012survey}. Recently, MCTS has found applications in enhancing the reasoning capabilities of Large Language Models (LLMs) \cite{yao2023tree, zhou24r}, enabling more structured and coherent text generation. However, as environments become increasingly complex and the cost of sampling increases, particularly in the context of LLMs where each node expansion may involve expensive model queries \cite{yao2023tree}, there is a growing need for more sample-efficient methods that can make better decisions with fewer simulations.

To address this challenge, we introduce DR-MCTS, a novel algorithm that integrates Doubly Robust (DR) off-policy estimation into the MCTS framework. By leveraging the strengths of both MCTS and DR estimation, our approach aims to improve sample efficiency and decision quality in complex environments. DR-MCTS represents a significant advancement in planning algorithms, particularly for scenarios where sample efficiency is crucial, such as in applications involving large-scale language models \cite{patterson2021carbon}.

Our work makes the following key contributions:

\begin{enumerate}
    \item We introduce DR-MCTS, a novel algorithm that incorporates Doubly Robust off-policy estimation into the MCTS framework.
    \item We formulate a hybrid estimator that combines traditional MCTS rollouts with DR estimation, and provide theoretical guarantees for unbiasedness and demonstrate variance reduction properties.
    \item We conduct extensive empirical evaluations in both the fully observable domain of Tic-Tac-Toe and the partially observable VirtualHome environment, demonstrating the superior performance of DR-MCTS compared to standard MCTS and MCTS with Importance Sampling.
    \item We perform a scaling analysis to show that DR-MCTS achieves better sample efficiency, particularly in scenarios with large language models.
\end{enumerate}

\subsection{Related Work}

Our work builds upon and extends research in several areas, including Monte Carlo Tree Search and off-policy evaluation in reinforcement learning.

Monte Carlo Tree Search (MCTS) \cite{coulom2006efficient} has become a cornerstone in reinforcement learning, leading to breakthroughs in AI for games and beyond. Notable implementations include AlphaGo \cite{silver2016mastering}, AlphaZero \cite{silver2018general}, and MuZero \cite{schrittwieser2020mastering}, which have demonstrated MCTS's effectiveness in complex decision-making tasks.

Recent work has focused on enhancing MCTS's sample efficiency. Borges and Oliveira \cite{borges2021combiningonpolicytrainingmodelbased} proposed combining off-policy and on-policy training in MCTS by deriving off-policy targets from the search tree. Their approach improved data utilization within the tree and showed promising results across various environments.

The intersection of MCTS and off-policy evaluation presents a promising avenue for improvement. MCTS inherently generates off-policy data through its exploration process. These sources of off-policy data within MCTS present opportunities for more efficient learning when coupled with appropriate off-policy evaluation techniques. This connection motivates the integration of advanced off-policy estimators into the MCTS framework to enhance its efficiency and accuracy.

Doubly robust (DR) estimation, originating from causal inference and biostatistics \cite{robins1995semiparametric}, has become a powerful tool in reinforcement learning (RL) for off-policy evaluation. While importance sampling \cite{precup2000eligibility} provided a foundational approach, it often suffers from high variance. DR estimation addresses this limitation by combining importance sampling with direct method estimation.

Significant advancements in DR estimation have expanded its applicability in reinforcement learning. Dudik et al. \cite{dudik2011doubly} introduced DR to contextual bandits, while Jiang and Li \cite{jiang2016doubly} extended it to full RL domains. Further refinements include weighted DR estimators for improved sample efficiency \cite{thomas2016data}, the More Robust Doubly Robust (MRDR) estimator for variance minimization \cite{farajtabar2018more}, and double reinforcement learning for efficient off-policy evaluation in both state and action spaces \cite{kallus2020double}.

Our DR-MCTS method distinguishes itself from previous works by directly incorporating the Doubly Robust estimator into the MCTS framework. Unlike approaches that use basic off-policy statistics gathered by MCTS \cite{borges2021combiningonpolicytrainingmodelbased} or apply advanced off-policy estimators only in standard RL settings \cite{jiang2016doubly, thomas2016data, farajtabar2018more, kallus2020double}, DR-MCTS leverages the statistical power of DR estimation within the tree search process. This novel integration aims to enhance sample efficiency and decision quality in complex, partially observable environments, with particular emphasis on scenarios involving large language models.

The rest of this paper is organized as follows: Section \ref{methods} covers the background on MCTS, importance sampling, and doubly robust estimation, and introduces our DR-MCTS algorithm with its theoretical properties. Section \ref{experiments} describes our experimental setup for Tic-Tac-Toe and VirtualHome environments and presents the results. Section \ref{discussion} analyzes our findings, including DR-MCTS's performance with smaller language models, and discusses limitations and future work. Section \ref{conclusion} summarizes our contributions.

\section{Methods}
\label{methods}
\subsection{Preliminaries}

This section provides an overview of the key concepts underlying our work: Monte Carlo Tree Search (MCTS), Importance Sampling (IS), and Doubly Robust (DR) estimation.

\subsubsection{Monte Carlo Tree Search}

Monte Carlo Tree Search (MCTS) is a heuristic search algorithm for decision processes, particularly effective in domains with large state spaces \cite{browne2012survey}. MCTS builds and updates a search tree iteratively, with each iteration consisting of four steps: selection, expansion, simulation, and backpropagation.

In the selection step, the algorithm traverses the tree from the root to a leaf node using a tree policy. While traditional MCTS often uses the Upper Confidence Bounds for Trees (UCT) algorithm, in our implementation, we employ the Polynomial Upper Confidence Trees (PUCT) algorithm \cite{silver2017mastering}, which is defined as:

\begin{equation}
\label{eqn:PUCT}
    a^* = \text{argmax}_a \left(Q(h,a) + c \pi_b(a|h) \frac{\sqrt{N(h)}}{1 + N(h,a)}\right)
\end{equation}

where $Q(h,a)$ is the estimated value of action $a$ given history $h$, $\pi_b(a|h)$ is the behavior policy, $N(h)$ is the number of visits to the node with history $h$, $N(h,a)$ is the number of times action $a$ was selected in the node with history $h$, and $c$ is an exploration constant. In environments with large action spaces, PUCT's use of the behavior policy helps focus the search on promising actions more quickly than UCT's purely visit-count-based exploration. This often results in improved sample efficiency, allowing PUCT to find good solutions with fewer simulations than UCT, which is particularly beneficial in computationally expensive environments.

The expansion step adds one or more child nodes to the selected leaf node. In the simulation step, a rollout is performed from the new node(s) to estimate the value of the state. The value estimate $V_{MCTS}(h)$ for a given history h is typically calculated as the average of the rewards obtained from all simulations that pass through the node with history $h$:

\begin{equation}
\label{eqn:V_MCTS}
    V_\text{MCTS}(h) = \frac{1}{N(h)} \sum_{i=1}^{N(h)} R_i(h)
\end{equation}

where $R_i(h)$ is the cumulative reward of the i-th simulation that passes through the node with history h, and $N(h)$ is the total number of such simulations.

Finally, the backpropagation step updates the statistics of the nodes in the path from the expanded node to the root, including the visit counts $N(h)$ and $N(h,a)$, and the value estimates $Q(h,a)$.

\subsection{Importance Sampling}

Importance Sampling (IS) is a technique used in off-policy evaluation to estimate the expected return of a target policy using data collected from a different behavior policy \cite{precup2000eligibility}. We present two forms of the IS estimator: the basic trajectory-wise IS estimator and an improved step-wise version \cite{jiang2016doubly}.

Let $\rho_t := \frac{\pi_e(a_t|h_t)}{\pi_b(a_t|h_t)}$ be the per-step importance ratio, where $\pi_e$ is the target (evaluation) policy and $\pi_b$ is the behavior policy. Define the cumulative importance ratio as $\rho_{1:t} := \prod_{k=1}^t \rho_k$. Then, for a trajectory $(h_0, a_0, r_0, \ldots, h_H)$, the estimators are defined as follows:

\begin{equation}
    V_\text{IS}(h) := \rho_{1:H} \cdot \left(\sum_{t=0}^{H-1} \gamma^t r_t\right)
\end{equation}

\begin{equation}
\label{eqn:step-IS}
    V_\text{step-IS}(h) := \sum_{t=0}^{H-1} \gamma^t \rho_{1:t} r_t
\end{equation}

The step-wise IS estimator ($V_\text{step-IS}$) often achieves lower variance than the basic IS estimator ($V_\text{IS}$), especially for longer trajectories. This is because it allows for partial usage of a trajectory: the cumulative importance ratio $\rho_{1:t}$ only accounts for the portion of the trajectory up to time $t$.

\subsection{Doubly Robust Estimation}

Doubly Robust (DR) estimation combines direct method (DM) estimation with importance sampling to provide an estimator that is unbiased and potentially has lower variance than pure IS \cite{jiang2016doubly}. The DR estimator is defined as:

\begin{equation}
\label{eqn:DR}
    V_\text{DR}(h) = \hat{V}(h) + \sum_{t=0}^{H-1} \gamma^t \rho_{1:t} (r_t + \gamma \hat{V}(h_{t+1}) - \hat{Q}(h_t, a_t))
\end{equation}

where $\rho_{1:t} = \prod_{k=1}^t \frac{\pi_e(a_k|h_k)}{\pi_b(a_k|h_k)}$ is the cumulative importance ratio, $\hat{V}(h)$ is an estimate of the value function, and $\hat{Q}(h_t, a_t)$ is an estimate of the action-value function.

The DR estimator has the desirable property of being unbiased if either the importance sampling weights are correct or if the value function estimates are accurate. This "double" protection against misspecification is what gives the estimator its name and makes it particularly useful in practical applications where perfect models are rare \cite{thomas2016data}.

A key aspect of DR estimation, as established in the seminal work of \citet{jiang2016doubly}, is the use of cross-validation for estimating $\hat{Q}(h_t, a_t)$. It helps reduce bias by mitigating overfitting, as using the same data to estimate $\hat{Q}$ and evaluate the DR estimator can lead to biased estimates \cite{Chernozhukov2018}. In reinforcement learning settings where data can be scarce and expensive to obtain, cross-validation allows for more efficient use of the available data by reusing it for both model fitting and evaluation \cite{arlot2010survey}.

\subsection{Doubly Robust MCTS}

We propose a hybrid estimator that combines the standard MCTS rollout estimate with the Doubly Robust (DR) estimate, leveraging the strengths of both MCTS and off-policy evaluation techniques. Our hybrid estimator is defined as:

\begin{equation}
\label{eqn:hybrid}
    V_{\text{hybrid}}(h) = \beta V_{\text{MCTS}}(h) + (1-\beta) V_{DR}(h)
\end{equation}

where $\beta \in [0,1]$ is a weighting parameter, $V_{\text{MCTS}}(h)$ is the standard MCTS rollout estimate, and $V_{DR}(h)$ is the DR estimate.

A crucial aspect of our DR-MCTS algorithm is the estimation of target and behavior policies. The target policy $\pi_e$ is computed using a softmax function over the Q-values of the children of a given state node:

\begin{equation}
\label{eqn:target_policy}
    \pi_e(a|h) = \frac{\exp(Q(h,a)/\tau)}{\sum_{a'} \exp(Q(h,a')/\tau)}
\end{equation}

where $Q(h,a)$ is the Q-value of action $a$ in history $h$, and $\tau$ is a temperature parameter controlling the exploration-exploitation trade-off. The behavior policy $\pi_b$ varies depending on the specific environment, as detailed in Appendix \ref{appendix:behavior_policy}.

In our DR-MCTS implementation, we estimate $\hat{V}(h)$ using a weighted average of rewards from child nodes, where weights are derived from the target policy:

\begin{equation}
    \hat{V}(h) = \sum_{a} \pi_e(a|h) \cdot \frac{1}{N(h,a)} \sum_{i=1}^{N(h,a)} R_i(h,a)
\end{equation}

where $N(h,a)$ is the number of times action $a$ was taken in history $h$, and $R_i(h,a)$ is the i-th observed reward for taking action $a$ in history $h$.

We estimate $\hat{Q}(h_t, a_t)$ using k-fold cross-validation on the rewards collected for each action:

\begin{equation}
\label{eqn:q_hat}
    \hat{Q}(h_t, a_t) = \frac{1}{K} \sum_{k=1}^K \frac{1}{|D_k|} \sum_{i \in D_k} R_i(h_t, a_t)
\end{equation}

where $K$ is the number of folds, and $D_k$ is the set of indices for the k-th fold. When insufficient data is available for k-fold cross-validation, we use the mean of all available rewards for the action. This approach helps to reduce overfitting and provides more robust estimates, which is particularly important in MCTS where the number of visits to each state-action pair can vary significantly.

The weighting parameter $\beta$ in our hybrid estimator plays a crucial role in balancing the contributions of the MCTS rollout estimate and the DR estimate. We explored a range of $\beta$ values from $[0, 0.25, 0.35, 0.5]$, and chose the one with best performance. 

%\subsubsection{Theoretical justification}
To establish the theoretical foundations of our hybrid estimator (Equation \ref{eqn:hybrid}), we present two key theorems. These theorems demonstrate that our proposed hybrid approach preserves the desirable properties of the Doubly Robust (DR) estimator (Equation \ref{eqn:DR}) while potentially offering additional benefits. Specifically, we prove the unbiasedness of the hybrid estimator and establish conditions under which it achieves variance reduction compared to standard MCTS.

\begin{theorem}[Unbiasedness of Hybrid Estimator]
\label{theorem1}
The hybrid estimator is unbiased for estimating the value of the target policy $\pi_e$.
\end{theorem}

\textit{Proof.} See Appendix \ref{appendiexA1} for the detailed proof.

\textbf{Implication:} Theorem \ref{theorem1} ensures that our hybrid estimator maintains the crucial property of unbiasedness, guaranteeing that the value estimates produced by our method will, on average, correctly represent the true values of states and actions.

\begin{theorem}[Variance Reduction Condition for Hybrid Estimator]
\label{theorem2}
Let $V_{\text{hybrid}}(h)$ be the hybrid estimator as defined in Equation \ref{eqn:hybrid}, and $V_{\text{MCTS}}(h)$ be the standard MCTS estimator. The hybrid estimator has lower variance than the standard MCTS estimator when:

\begin{equation}
\mathbb{E}\left[\sum_{t=0}^{H-1} \gamma^{2t} \rho_{1:t}^2 (Q(h_t, a_t) - \hat{Q}(h_t, a_t))^2\right] = o\left(\frac{\text{Var}(V_{\text{MCTS}}(h))}{(1-\beta)^2}\right)
\end{equation}

where $\beta \in (0,1)$ is the weighting parameter in the hybrid estimator, $\gamma$ is the discount factor, $\rho_{1:t}$ is the cumulative importance ratio, and $Q(h_t, a_t)$ and $\hat{Q}(h_t, a_t)$ are the true and estimated action-value functions, respectively.
\end{theorem}
\textit{Proof.} See Appendix \ref{appendixA2} for the detailed proof.

\textbf{Implication:} Theorem \ref{theorem2} provides a precise condition for variance reduction in our hybrid estimator compared to standard MCTS. The variance reduction is maximized when $\hat{Q}$ closely approximates $Q$, while still allowing for some discrepancy. However, our hybrid estimator is robust and can provide benefits even with imperfect $Q$-value estimates. As long as the $Q$-value estimation errors are sufficiently small relative to MCTS variance, our method ensures more reliable value estimates across various scenarios. This property potentially leads to faster convergence, more robust decision-making, and improved sample efficiency, even in challenging environments where perfect $Q$-value estimation is infeasible.

%\subsubsection{DR-MCTS Algorithm}
\begin{algorithm}[tb]
   \caption{DR-MCTS Algorithm}
   \label{alg:drmcts}
\begin{algorithmic}
   \STATE {\bfseries Input:} state $s_0$, history $h_0$, iterations $N$, estimator type
   \STATE Initialize tree $T$ with root $(s_0, h_0)$
   \FOR{$i=1$ {\bfseries to} $N$}
   \STATE $(s, h) \gets (s_0, h_0)$
   \WHILE{$(s, h)$ is not terminal {\bfseries and} $(s, h)$ is in $T$}
   \STATE $a \gets \text{argmax}_{a'} \text{PUCT}((s,h), a')$
   \STATE $s, h \gets \text{Apply}(s, a), h + a$
   \ENDWHILE
   \IF{$(s, h)$ is not terminal}
   \STATE Add $(s, h)$ to $T$
   \STATE $v_{\text{MCTS}} \gets \text{Simulate}(s, h)$
   \IF{estimator type = DR}
   \STATE $v_{\text{DR}} \gets \text{ComputeDR}(s, h, \pi_e, \pi_b, \hat{Q}, \hat{V})$
   \ENDIF
   \STATE $v \gets \beta v_{\text{MCTS}} + (1-\beta) v_{\text{DR}}$ %\COMMENT{Hybrid Estimator}
   \ELSE
   \STATE $v \gets \text{Reward}(s)$
   \ENDIF
   \WHILE{$(s, h)$ is not $(s_0, h_0)$}
   \STATE Update statistics for $(s, h)$ in $T$ with $v$
   \STATE $(s, h) \gets \text{Parent}(s, h)$
   \ENDWHILE
   \ENDFOR
   \STATE {\bfseries Return:} $\text{argmax}_a Q((s_0, h_0), a)$
\end{algorithmic}
\end{algorithm}

We now introduced the algorithm we used for DR-MCTS (Algorithm \ref{alg:drmcts}). Our algorithm initializes a search tree with the root node representing the initial state and history. For a specified number of iterations, it traverses the tree using the PUCT selection strategy (Equation \ref{eqn:PUCT}), balancing exploration and exploitation. When a new node is reached, it's added to the tree, and a simulation estimates its value $V_{MCTS}$. If the DR estimator is used, $V_{DR}$ is calculated (Equation \ref{eqn:DR}). These estimates are combined using the hybrid estimator (Equation \ref{eqn:hybrid}). The resulting value is backpropagated, updating node statistics. After all iterations, the algorithm returns the action with the highest estimated value at the root node.

\section{Experiments}
\label{experiments}
\subsection{Experimental Setup}

To evaluate the performance of our DR-MCTS algorithm, we conducted experiments in two distinct environments: the classic game of Tic-Tac-Toe and the more complex VirtualHome environment. These environments were chosen to test the algorithm's effectiveness in both simple, fully observable domains and complex, partially observable scenarios.

\subsubsection{Tic-Tac-Toe}

\textbf{Tic-Tac-Toe.} To evaluate our DR-MCTS algorithm, we first conducted experiments in the classic game of Tic-Tac-Toe. Tic-Tac-Toe is played on a 3x3 grid where two players alternately place their symbols (X and O). The game concludes when one player forms a line of three symbols horizontally, vertically, or diagonally, or when all cells are filled, resulting in a draw.

\textbf{Data.} We implemented a Tic-Tac-Toe environment in Python to generate game data for our experiments. The game state is represented by a list of 9 elements, each corresponding to a cell on the board, which can be empty, `X', or `O'. Our environment tracks the current player and provides methods for making moves, checking for valid moves, determining game outcomes, and cloning the game state. The reward structure assigns 1.0 for a win, 0.5 for a draw, and 0.0 for non-terminal moves. Invalid moves, while accounted for in the implementation with a -1.0 reward, do not occur in actual gameplay due to the constraints imposed by the action selection process.

\textbf{Evaluation.} Our evaluation framework compared three algorithms: standard MCTS, MCTS with Importance Sampling (IS-MCTS), and our proposed DR-MCTS. We conducted experiments across a range of MCTS rollouts (20, 40, 60, 80, and 100) to assess performance scaling with increased computational resources. For each algorithm pair and rollout count, we simulated 100 games, alternating the roles of X and O to ensure fairness. The primary performance metric was the win rate of each algorithm against its opponent, with draws recorded separately. 

\textbf{Baselines.} We used two baseline models for comparison. The first was standard MCTS, which uses pure Monte Carlo rollouts to estimate state values. The second was IS-MCTS, an intermediate comparison point that incorporates the step-wise importance sampling estimator. That is, we replace $v_{DR}$ in Algorithm \ref{alg:drmcts} with equation \ref{eqn:step-IS}. Both baselines use the same tree structure and selection strategy as our DR-MCTS algorithm, differing only in their value estimation methods.

\subsubsection{Virtual Home}
\textbf{Virtual Home. } Virtual Home is a complex, partially observable 3D environment simulating a household setting \cite{puig2018virtualhome}. It contains multiple rooms (e.g., kitchen, living room, bedroom) with hundreds of interactive items and containers. The state space is high-dimensional and partially observable, while the action space is large, including navigation and object interaction actions. Following Zhao et al. 2023 \cite{Zhao2023}, we utilize Large Language Models (LLMs) in two crucial roles: as a world model and as a policy model. The LLM-based world model informs the planning process by providing commonsense knowledge about the environment, which is used to estimate Q-values. Additionally, an LLM-based policy model is employed to generate a heuristic policy, guiding action selection during the search process. This dual use of LLMs enables more efficient exploration of the vast state-action space in Virtual Home by leveraging the models' understanding of typical household layouts and object interactions.

\textbf{Data.} We create tasks with randomly initialized scenes and expert trajectories. We evaluate on four datasets, each testing different aspects of generalization:

1) \textit{Novel Simple}: Rearranging familiar items in new ways, e.g.,``put one plate inside the fridge'' when training only included ``put one plate on the kitchen table".

2) \textit{Novel Objects}: Tasks with unfamiliar objects, e.g.,``place the blender on the kitchen counter" when blenders were not in training data.

3) \textit{Novel Comp.}: New combinations of familiar tasks, e.g.,``put the fork on the dining table and the milk in the fridge" when training only had these actions separately.

4) \textit{Novel Comp. + Objects}: Combining new task compositions with unfamiliar objects, e.g.,``place the juicer on the kitchen counter and the kiwi in the fruit bowl," where neither item appeared in training data.

\textbf{Evaluation.} We evaluate DR-MCTS against standard MCTS and IS-MCTS in two parts: success rate comparison and inference time scaling. For success rates, we use LLMs (GPT-4o-mini and GPT-4o) as world and policy models across all four tasks. Success is defined as task completion within 10 steps. For scaling analysis, we compare two configurations: (1) GPT-4o for both world and policy models in MCTS, and (2) a hybrid approach (GPT-4o-mini as world model, GPT-4o as policy model) for all algorithms. We vary search tree depth from 5 to 20 steps, running 50 episodes per task with a fixed random seed. 

\textbf{Baselines.} We used the same baselines as Tic-Tac-Toe. Comparing DR-MCTS against standard MCTS (using pure Monte Carlo rollouts) and IS-MCTS. All algorithms use the same tree structure and PUCT selection strategy, differing only in value estimation methods. 

\subsection{Results}
\paragraph{Tic-Tac-Toe} Figure \ref{fig:1} and Table \ref{tab:tic-tac-toe} show the win rates of DR-MCTS and IS-MCTS compared to standard MCTS in Tic-Tac-Toe, across different rollout numbers. Both DR-MCTS and IS-MCTS consistently outperform MCTS, with DR-MCTS showing a more pronounced advantage. DR-MCTS achieves win rates ranging from 63\% to 88\%, while IS-MCTS wins between 57\% and 82\% of games. The performance gap widens as the number of rollouts increases, with DR-MCTS reaching an 88\% win rate at 100 rollouts compared to MCTS's 10\%. These results demonstrate the superior performance of both DR-MCTS and IS-MCTS over standard MCTS in Tic-Tac-Toe, with DR-MCTS exhibiting the strongest overall performance.
\begin{figure}[t]
\vskip 0.2in
    \centering
    \includegraphics[width=0.5\textwidth]{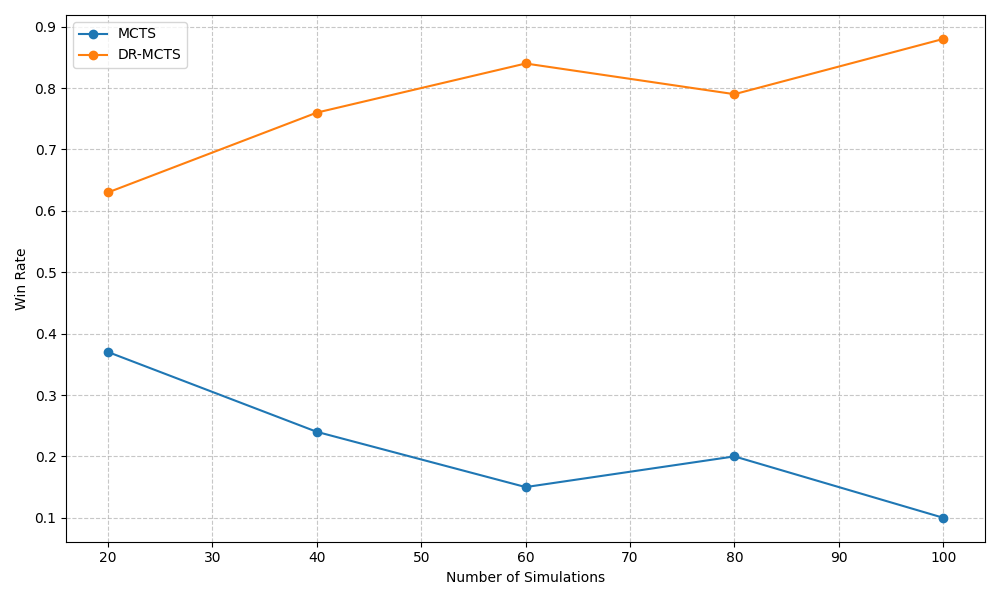}
    %\caption{DR-MCTS against MCTS}
    \label{fig:1a}
    
    \vspace{0.1cm}
    
    \includegraphics[width=0.5\textwidth]{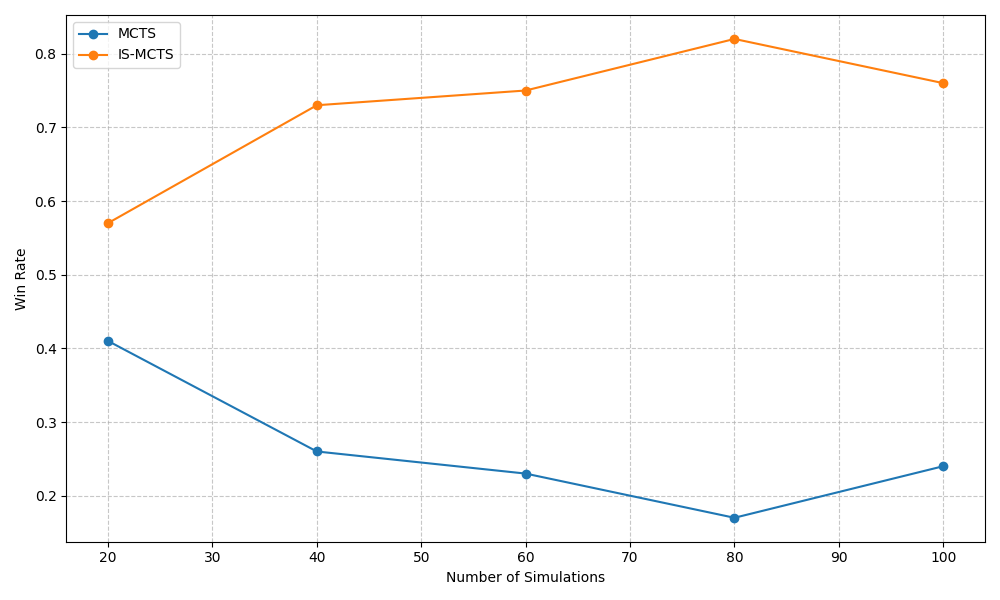}
    %\caption{IS-MCTS against MCTS}
    \label{fig:1b}
    
    \caption{Win rates for Tic-Tac-Toe across number of rollouts. Top: DR-MCTS vs. MCTS. Bottom: IS-MCTS vs. MCTS}
    \label{fig:1}
\vskip -0.2in
\end{figure}

\begin{table}[t]
\caption{Win rates comparison of MCTS vs DR-MCTS and MCTS vs IS-MCTS for Tic-Tac-Toe across different rollout numbers.}
\label{tab:tic-tac-toe}
\vskip 0.15in
\begin{center}
\begin{small}
\begin{sc}
\begin{tabular}{lccc}
\toprule
Rollouts & MCTS & DR-MCTS & IS-MCTS \\
\midrule
20  & 0.37 & \textbf{0.63} & - \\
40  & 0.24 & \textbf{0.76} & - \\
60  & 0.15 & \textbf{0.84} & - \\
80  & 0.20 & \textbf{0.79} & - \\
100 & 0.10 & \textbf{0.88} & - \\
\midrule
20  & 0.41 & - & \textbf{0.57} \\
40  & 0.26 & - & \textbf{0.73} \\
60  & 0.23 & - & \textbf{0.75} \\
80  & 0.17 & - & \textbf{0.82} \\
100 & 0.24 & - & \textbf{0.76} \\
\bottomrule
\end{tabular}
\end{sc}
\end{small}
\end{center}
\vskip -0.1in
\end{table}

\paragraph{Virtual Home} 
\begin{figure}[t]
\vskip 0.2in
    \centering
    \includegraphics[width=0.5\textwidth]{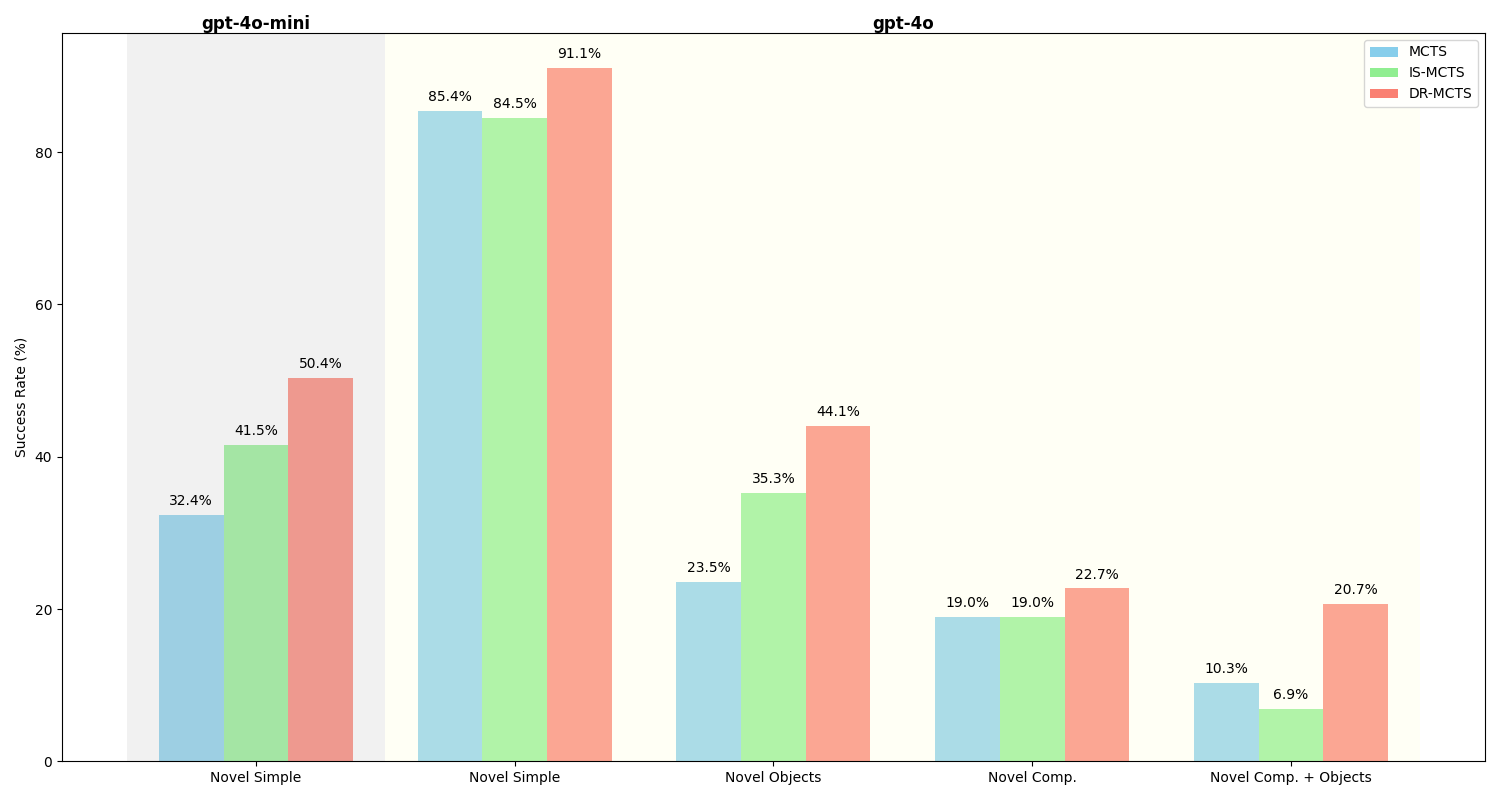}
    \caption{Success rates for Virtual Home tasks across different algorithms and model sizes}
    \label{fig:virtualhome_success}
    \vskip -0.2in
\end{figure}

\begin{table*}[t]
\caption{Performance comparison of different estimators across task types and LLM models}
\label{tab:performance_comparison}
\vskip 0.15in
\centering
\resizebox{\textwidth}{!}{%
\begin{tabular}{llllrrrrrr}
\toprule
LLM model & Estimator & Tasks & Success rate & Input Tokens (M) & Output Tokens (M) & API cost (\$) & Time \\
\midrule
gpt-4o-mini & MCTS & Novel Simple & 32.40\% & 29.24 & 0.823 & 4.88 & 15:43:51 \\
gpt-4o-mini & IS-MCTS & Novel Simple & 41.50\% & 28.38 & 0.746 & 4.71 & 14:28:27 \\
gpt-4o-mini & DR-MCTS & Novel Simple & \textbf{50.40\%} & \textbf{27.31} & \textbf{0.648} & \textbf{4.49} & \textbf{13:41:15} \\
\midrule
gpt-4o & MCTS & Novel Simple & 85.40\% & 16.402 & 0.268 & 44.69 & 8:43:53 \\
gpt-4o & IS-MCTS & Novel Simple & 84.50\% & 16.323 & \textbf{0.262} & 43.83 & 7:33:40 \\
gpt-4o & DR-MCTS & Novel Simple & \textbf{91.10\%} & \textbf{16.053} & 0.263 & \textbf{43.13} & \textbf{7:24:42} \\
\midrule
gpt-4o & MCTS & Novel Objects & 23.50\% & 8.299 & 0.126 & 22.00 & 4:32:02 \\
gpt-4o & IS-MCTS & Novel Objects & 35.30\% & 8.565 & 0.133 & 22.74 & 3:57:00 \\
gpt-4o & DR-MCTS & Novel Objects & \textbf{44.10\%} & \textbf{8.03} & \textbf{0.117} & \textbf{21.24} & \textbf{3:50:33} \\
\midrule
gpt-4o & MCTS & Novel Comp. & 19.00\% & \textbf{2.107} & \textbf{0.034} & \textbf{5.61} & 3:13:29 \\
gpt-4o & IS-MCTS & Novel Comp. & 19.00\% & 2.218 & 0.035 & 5.90 & 3:05:52 \\
gpt-4o & DR-MCTS & Novel Comp. & \textbf{22.70\%} & 2.489 & 0.038 & 6.60 & \textbf{2:48:51} \\
\midrule
gpt-4o & MCTS & Novel Comp + Objects & 10.34\% & \textbf{6.279} & \textbf{0.098} & \textbf{16.68} & 3:24:29 \\
gpt-4o & IS-MCTS & Novel Comp + Objects & 6.89\% & 6.763 & 0.107 & 17.98 & \textbf{2:39:38} \\
gpt-4o & DR-MCTS & Novel Comp + Objects & \textbf{20.68\%} & 7.501 & 0.118 & 19.94 & 3:05:33 \\
\bottomrule
\end{tabular}%
}
\begin{small}
\parbox{\textwidth}{
\textit{Note:} Bold values indicate the best performance for each metric: highest for success rate; lowest for input tokens, output tokens, API cost, and time.
}
\end{small}
\vskip -0.1in
\end{table*}

\begin{figure}[!htbp]
\vskip 0.2in
    \centering
    \includegraphics[width=0.5\textwidth]{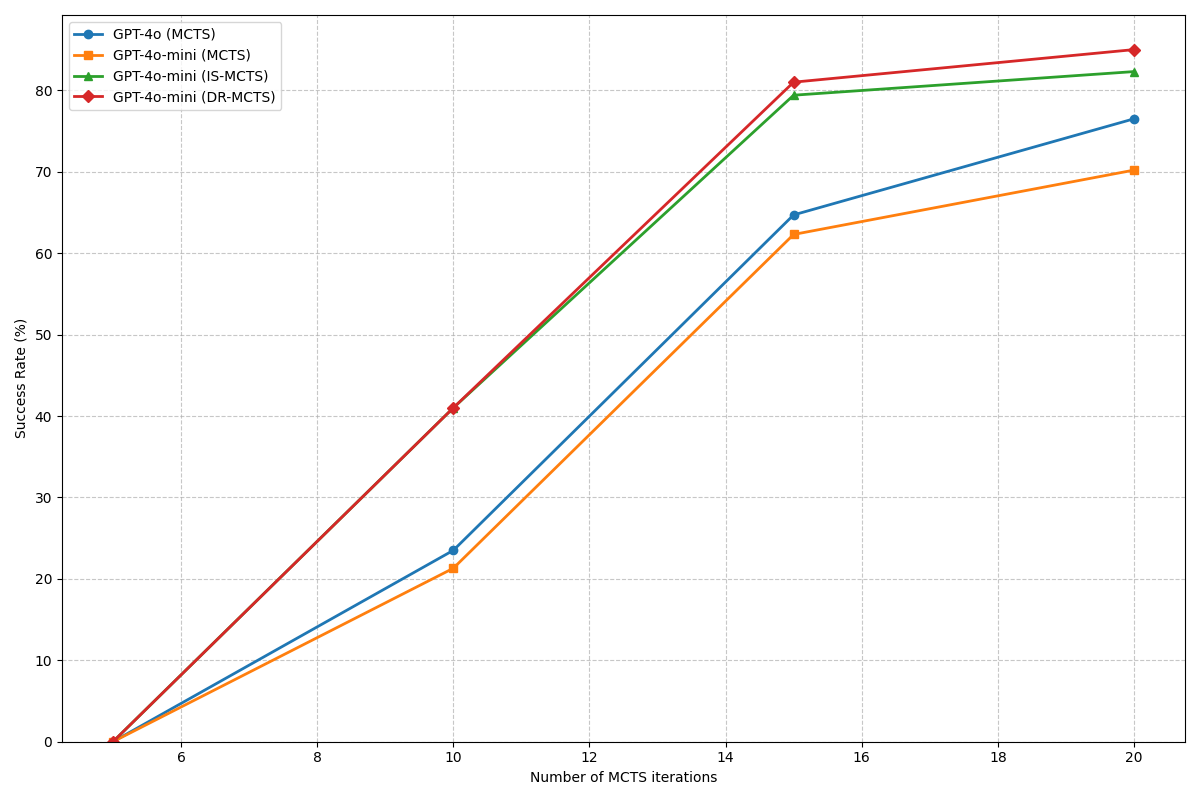}
    \vspace{0.5em}
    \includegraphics[width=0.5\textwidth]
    {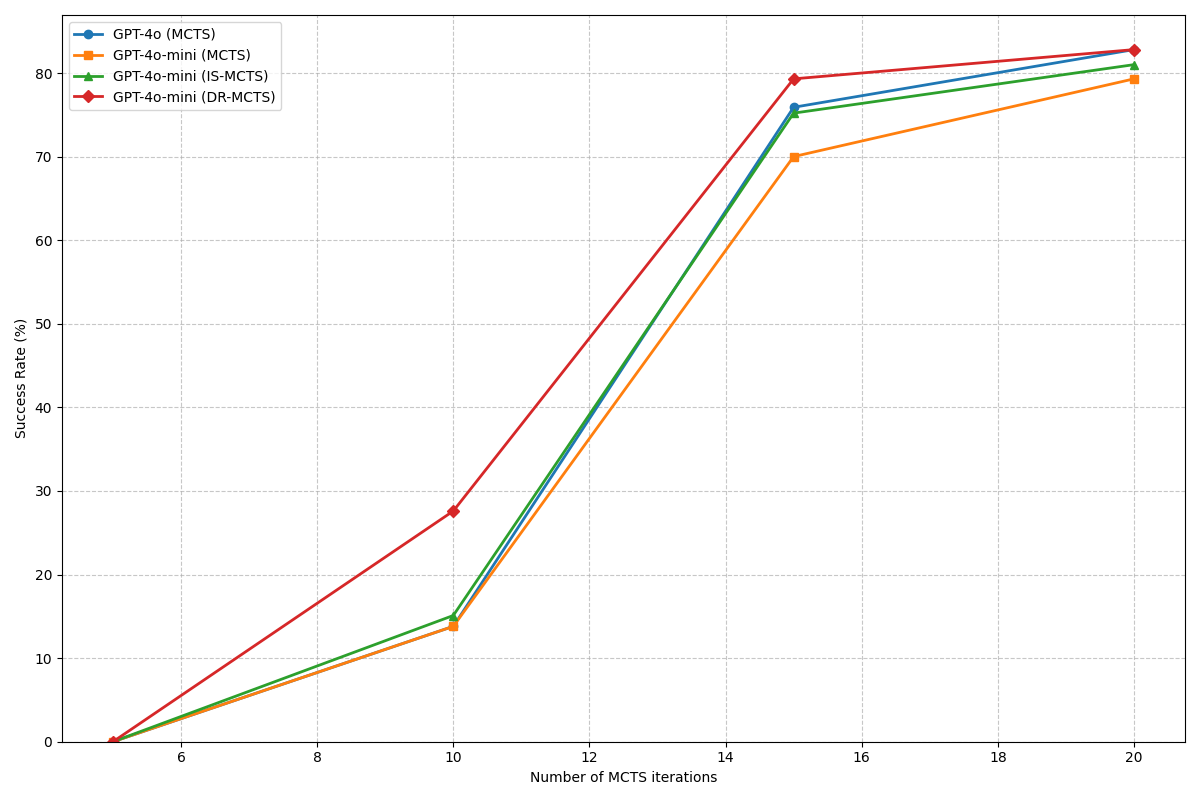}
    \vspace{0.5em}
    \includegraphics[width=0.5\textwidth]{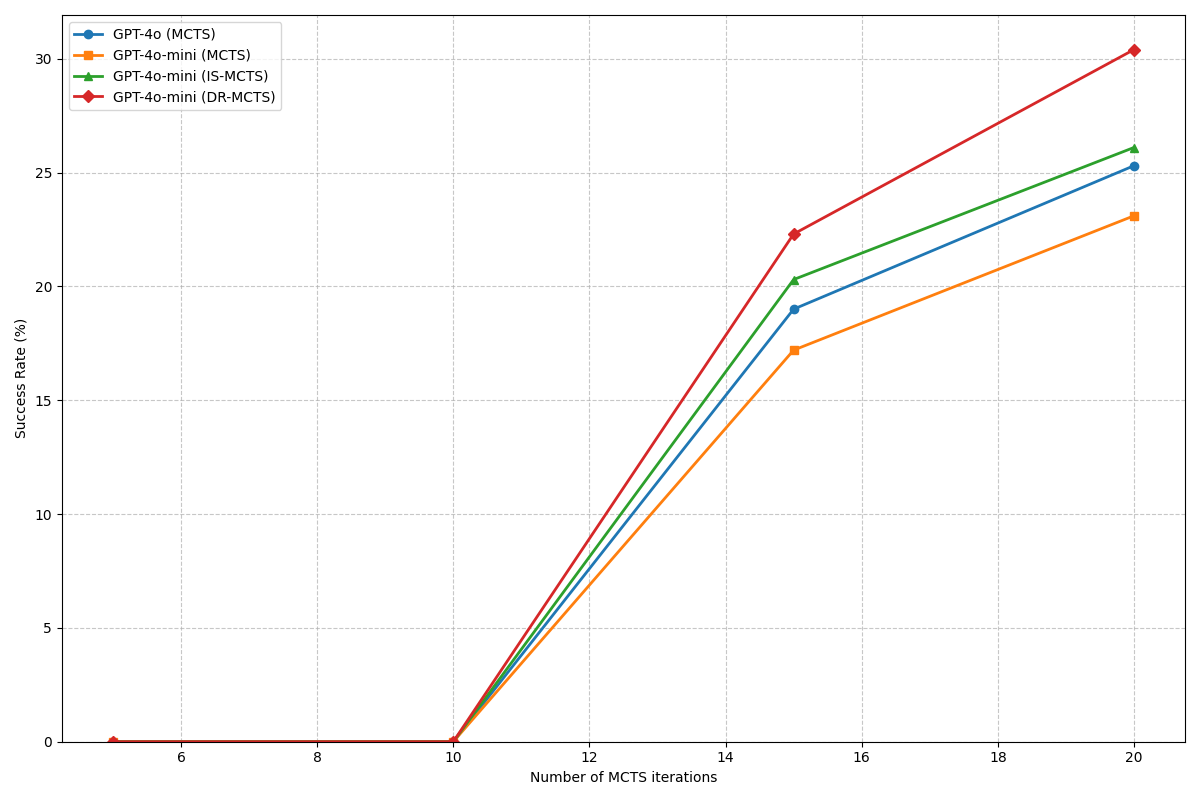}
    \caption{Scaling analysis for different task types in Virtual Home. Top: Novel Objects tasks. Middle: Novel Comp. tasks. Bottom: Novel Comp. + Objects tasks.}
    \label{fig:scaling_analysis}
        \vskip -0.2in
\end{figure}

Table \ref{tab:performance_comparison} and Figure \ref{fig:virtualhome_success} presents the performance of MCTS, IS-MCTS, and DR-MCTS across different task types and LLM models. DR-MCTS consistently outperforms the other methods in terms of success rate across all task types. For Novel Simple tasks, DR-MCTS achieves a 50.4\% success rate with GPT-4o-mini and 91.1\% with GPT-4o, compared to 32.4\% and 85.4\% for standard MCTS, respectively. This performance advantage is maintained for more complex tasks, with DR-MCTS achieving 44.1\% success on Novel Objects tasks and 22.7\% on Novel Comp. tasks, compared to 23.5\% and 19.0\% for MCTS, respectively.

Interestingly, DR-MCTS exhibits a task-dependent pattern in resource utilization. For simpler tasks, it generates fewer tokens and requires less computation time while achieving higher success rates. For instance, in Novel Simple tasks with GPT-4o, DR-MCTS uses 16.053M input tokens and takes 7:24:42, compared to 16.402M input tokens and 8:43:53 for MCTS, while achieving a higher success rate (91.1\% vs 85.4\%). As task complexity increases, DR-MCTS tends to generate more tokens to achieve its superior performance. In Novel Comp. + Objects tasks, DR-MCTS uses 7.501M input tokens compared to 6.279M for MCTS to achieve a higher success rate (20.68\% vs. 10.34 \%). This suggests that DR-MCTS adapts its resource allocation based on task difficulty, investing computational effort more effectively to maintain high success rates across varying task complexities.

Figure \ref{fig:scaling_analysis} illustrates the scaling behavior of different algorithms as the number of MCTS search steps increases. DR-MCTS demonstrates superior performance across all search depths, with its advantage being particularly pronounced in the early stages and for complex tasks. This highlights DR-MCTS's efficiency in quickly identifying promising actions in challenging environments. Notably, DR-MCTS with GPT-4o-mini as the world model often outperforms standard MCTS with the larger GPT-4o model, especially at lower number of MCTS iterations. These findings indicate that DR-MCTS can achieve competitive performance with fewer computational resources and smaller models, offering crucial advantages in time-sensitive or resource-constrained scenarios and potentially leading to substantial computational savings in practical applications.

\section{Discussion}
\label{discussion}
Our experimental results demonstrate the significant advantages of DR-MCTS over traditional MCTS across different domains, from simple Tic-Tac-Toe to complex, partially observable VirtualHome environments. In Tic-Tac-Toe, DR-MCTS consistently outperformed standard MCTS, with win rates ranging from 63\% to 88\%, compared to MCTS's 10\% to 37\%. This performance gap was maintained in the more challenging VirtualHome tasks, highlighting DR-MCTS's ability to handle increased complexity and partial observability.

The superior performance of DR-MCTS can be understood through our theoretical findings on variance reduction. Our analysis shows that the hybrid estimator achieves lower variance than standard MCTS when $\mathbb{E}[\sum_{t=0}^{H-1} \gamma^{2t} \rho_{1:t}^2 (Q(h_t, a_t) - \hat{Q}(h_t, a_t))^2] = o(\text{Var}(V_{\text{MCTS}}(h))/(1-\beta)^2)$, a condition likely met in many practical scenarios, especially in complex environments where MCTS variance is typically high. Importantly, our theory suggests that variance reduction is maximized when $\hat{Q}$ closely approximates $Q$, while still allowing for some discrepancy. This insight aligns with our experimental observations, where DR-MCTS showed improvements across different environments, with particularly notable gains in complex, partially observable environments like VirtualHome.

The observed task-dependent resource utilization pattern further supports our theoretical insights. In simpler tasks, where Q-value estimation errors might be small, DR-MCTS achieves superior performance with minimal additional resources. In more complex tasks, it allocates more resources to capitalize on the greater potential for variance reduction. This adaptive behavior demonstrates how DR-MCTS balances the trade-off between Q-value estimation accuracy and MCTS exploration, as implied by our theoretical findings.  Moreover, DR-MCTS's ability to achieve superior performance using smaller language models as world models aligns with our variance reduction theorem. The hybrid estimator's tolerance for imperfect Q-value estimates allows it to leverage less precise but more computationally efficient models effectively. 

While DR-MCTS has demonstrated effectiveness, it faces several limitations that warrant further investigation. The fine-tuning process for the $\beta$ parameter in the hybrid estimator can be computationally expensive, particularly in LLM-based environments where API calls are costly. Additionally, our current method for estimating $\hat{V}$ and $\hat{Q}$ relies on empirical averages rather than a strictly model-based approach. While effective in our experiments, this technique may not fully capture the intricacies of environment dynamics, especially in highly stochastic or partially observable settings.

To address these limitations and extend the capabilities of DR-MCTS, we propose several promising future research directions. First, implementing a dynamic $\beta$ function that adapts based on the relative magnitudes of Q-value estimation errors and MCTS variance could potentially reduce the need for extensive parameter tuning and mitigate computational costs. Second, incorporating more rigorous, model-based estimation methods for $\hat{V}$ and $\hat{Q}$ could lead to improved performance. For instance, using neural networks to learn state-value functions or employing Gaussian processes for value function approximation might better handle sparse rewards and provide stronger theoretical convergence guarantees, particularly in complex environments requiring long-term planning. These enhancements could further optimize the balance between Q-value estimation accuracy and MCTS exploration, potentially leading to even greater performance gains in challenging domains.

\section{Conclusion}
\label{conclusion}
DR-MCTS represents a significant advancement in Monte Carlo Tree Search methods, demonstrating consistent improvements over traditional approaches across various domains and complexity levels. By leveraging doubly robust estimation, MCTS-DR achieves higher success rates, better sample efficiency, and improved scalability, particularly in challenging scenarios and under computational constraints. The ability of DR-MCTS to maintain high performance with smaller large language models opens up new possibilities for deploying sophisticated planning algorithms in resource-limited settings. As AI systems are increasingly applied to complex real-world problems, techniques like DR-MCTS that can effectively balance performance, efficiency, and adaptability will become increasingly valuable.

\section*{Impact Statement}

This paper presents advancements in Monte Carlo Tree Search (MCTS) algorithms, aiming to improve decision-making processes in complex, partially observable environments. While our primary goal is to contribute to the field of Machine Learning, we acknowledge that this work may have broader societal implications. Our DR-MCTS algorithm could potentially enhance planning and decision-making in various real-world applications, such as robotics, automated planning, and resource management. The improved sample efficiency and performance with smaller models may lead to more energy-efficient AI systems, contributing to reduced computational costs and environmental impact. However, as with any advanced AI technology, there is a risk of unintended consequences if deployed in sensitive domains without proper oversight. The improved decision-making capabilities could be misused if applied in ethically questionable contexts. Additionally, while our algorithm aims to make better decisions, it inherits any biases present in the underlying models or training data. Users should be aware of potential biases and take appropriate measures to ensure fair and ethical application. We encourage further research into the ethical implications of advanced planning algorithms and advocate for responsible development and deployment of such technologies.

\newpage
\bibliography{reference}
\bibliographystyle{icml2025}

\newpage
\appendix
\onecolumn
\section{Theoretical Analysis of DR-MCTS}

\subsection{Unbiasedness of the Hybrid Estimator}
\label{appendiexA1}

\begin{theorem}[Unbiasedness of Hybrid Estimator]
The hybrid estimator is unbiased for estimating the value of the target policy $\pi_e$.
\end{theorem}

\begin{proof}
We know that $V_{\text{MCTS}}(h)$ is unbiased due to the properties of Monte Carlo estimation. For $V_{DR}(h)$, we can express it as:

\begin{equation}
    V_{DR}(h) = \hat{V}(h) + \sum_{t=0}^{H-1} \gamma^t \rho_{1:t} (r_t + \gamma \hat{V}(h_{t+1}) - \hat{Q}(h_t, a_t))
\end{equation}

where $\rho_{1:t} = \prod_{k=1}^t \frac{\pi_e(a_k|h_k)}{\pi_b(a_k|h_k)}$ is the cumulative importance ratio.

Taking the expectation with respect to the behavior policy $\pi_b$:

\begin{align}
    \mathbb{E}_{\pi_b}[V_{DR}(h)] &= \mathbb{E}_{\pi_b}[\hat{V}(h)] + \mathbb{E}_{\pi_b}\left[\sum_{t=0}^{H-1} \gamma^t \rho_{1:t} (r_t + \gamma \hat{V}(h_{t+1}) - \hat{Q}(h_t, a_t))\right] \\
    &= \hat{V}(h) + \sum_{t=0}^{H-1} \gamma^t \mathbb{E}_{\pi_b}[\rho_{1:t} (r_t + \gamma \hat{V}(h_{t+1}) - \hat{Q}(h_t, a_t))] \\
    &= \hat{V}(h) + \sum_{t=0}^{H-1} \gamma^t (Q(h_t, a_t) - \mathbb{E}_{\pi_e}[\hat{Q}(h_t, a_t)]) \\
    &= V(h)
\end{align}

The last step follows from the fact that $\mathbb{E}_{\pi_e}[\hat{Q}(h_t, a_t)] = Q(h_t, a_t)$ when $\hat{Q}$ is an unbiased estimator of $Q$.

Therefore, both $V_{\text{MCTS}}(h)$ and $V_{DR}(h)$ are unbiased estimators of $V(h)$. Since $V_{\text{hybrid}}(h)$ is a linear combination of these unbiased estimators, it is also unbiased:

\begin{align}
    \mathbb{E}[V_{\text{hybrid}}(h)] &= \mathbb{E}[\beta V_{\text{MCTS}}(h) + (1-\beta) V_{DR}(h)] \\
    &= \beta \mathbb{E}[V_{\text{MCTS}}(h)] + (1-\beta) \mathbb{E}[V_{DR}(h)] \\
    &= \beta V(h) + (1-\beta) V(h) \\
    &= V(h)
\end{align}

Thus, the hybrid estimator remains unbiased in the DR-MCTS context.
\end{proof}

\subsection{Variance Reduction Condition for Hybrid Estimator}
\label{appendixA2}
\begin{theorem}[Variance Reduction Condition for Hybrid Estimator]
Let $V_{\text{hybrid}}(h)$ be the hybrid estimator as defined in Equation \ref{eqn:hybrid}, and $V_{\text{MCTS}}(h)$ be the standard MCTS estimator. The hybrid estimator has lower variance than the standard MCTS estimator when:

\begin{equation}
\mathbb{E}\left[\sum_{t=0}^{H-1} \gamma^{2t} \rho_{1:t}^2 (Q(h_t, a_t) - \hat{Q}(h_t, a_t))^2\right] = o\left(\frac{\text{Var}(V_{\text{MCTS}}(h))}{(1-\beta)^2}\right)
\end{equation}

where $\beta \in (0,1)$ is the weighting parameter in the hybrid estimator, $\gamma$ is the discount factor, $\rho_{1:t}$ is the cumulative importance ratio, and $Q(h_t, a_t)$ and $\hat{Q}(h_t, a_t)$ are the true and estimated action-value functions, respectively.
\end{theorem}

\begin{proof}
We begin by expressing the variance of the hybrid estimator:

\begin{align}
\text{Var}(V_{\text{hybrid}}(h)) &= \text{Var}(\beta V_{\text{MCTS}}(h) + (1-\beta) V_{DR}(h)) \\
&= \beta^2 \text{Var}(V_{\text{MCTS}}(h)) + (1-\beta)^2 \text{Var}(V_{DR}(h)) + 2\beta(1-\beta)\text{Cov}(V_{\text{MCTS}}(h), V_{DR}(h))
\end{align}

Let $\Delta(h) = V_{DR}(h) - V_{\text{MCTS}}(h)$. Then:

\begin{align}
\text{Var}(V_{\text{hybrid}}(h)) &= \text{Var}(V_{\text{MCTS}}(h)) + (1-\beta)^2 \text{Var}(\Delta(h)) + 2\beta(1-\beta)\text{Cov}(V_{\text{MCTS}}(h), \Delta(h))
\end{align}

We make the following assumptions:
\begin{enumerate}
    \item The DR estimator is unbiased, i.e., $\mathbb{E}[V_{DR}(h)] = V(h)$, where $V(h)$ is the true value function.
    \item The MCTS estimator is also unbiased, i.e., $\mathbb{E}[V_{\text{MCTS}}(h)] = V(h)$.
\end{enumerate}

Under these assumptions, we have $\mathbb{E}[\Delta(h)] = 0$. Now, let's consider the covariance term:

\begin{align}
\text{Cov}(V_{\text{MCTS}}(h), \Delta(h)) &= \mathbb{E}[V_{\text{MCTS}}(h)\Delta(h)] - \mathbb{E}[V_{\text{MCTS}}(h)]\mathbb{E}[\Delta(h)] \\
&= \mathbb{E}[V_{\text{MCTS}}(h)\Delta(h)] \quad \text{(since $\mathbb{E}[\Delta(h)] = 0$)}
\end{align}

Using the Cauchy-Schwarz inequality, we can bound this covariance:

\begin{align}
|\text{Cov}(V_{\text{MCTS}}(h), \Delta(h))| &\leq \sqrt{\text{Var}(V_{\text{MCTS}}(h)) \cdot \text{Var}(\Delta(h))}
\end{align}

Now, let's consider the variance of $\Delta(h)$. We can express $\Delta(h)$ as:

\begin{equation}
\Delta(h) = \sum_{t=0}^{H-1} \gamma^t \rho_{1:t} (r_t + \gamma \hat{V}(h_{t+1}) - \hat{Q}(h_t, a_t))
\end{equation}

Using Jensen's inequality, we can bound the variance of $\Delta(h)$:

\begin{align}
\text{Var}(\Delta(h)) &= \text{Var}\left(\sum_{t=0}^{H-1} \gamma^t \rho_{1:t} (r_t + \gamma \hat{V}(h_{t+1}) - \hat{Q}(h_t, a_t))\right) \\
&\leq \sum_{t=0}^{H-1} \text{Var}\left(\gamma^t \rho_{1:t} (r_t + \gamma \hat{V}(h_{t+1}) - \hat{Q}(h_t, a_t))\right) \\
&= \sum_{t=0}^{H-1} \gamma^{2t} \mathbb{E}[\rho_{1:t}^2] \cdot \text{Var}(r_t + \gamma \hat{V}(h_{t+1}) - \hat{Q}(h_t, a_t))
\end{align}

We can further bound this by noting that $\text{Var}(r_t + \gamma \hat{V}(h_{t+1}) - \hat{Q}(h_t, a_t)) \leq \mathbb{E}[(r_t + \gamma \hat{V}(h_{t+1}) - \hat{Q}(h_t, a_t))^2]$:

\begin{align}
\text{Var}(\Delta(h)) &\leq \sum_{t=0}^{H-1} \gamma^{2t} \mathbb{E}[\rho_{1:t}^2] \cdot \mathbb{E}[(r_t + \gamma \hat{V}(h_{t+1}) - \hat{Q}(h_t, a_t))^2] \\
&= \mathbb{E}\left[\sum_{t=0}^{H-1} \gamma^{2t} \rho_{1:t}^2 (r_t + \gamma \hat{V}(h_{t+1}) - \hat{Q}(h_t, a_t))^2\right]
\end{align}

Let $\epsilon := \mathbb{E}\left[\sum_{t=0}^{H-1} \gamma^{2t} \rho_{1:t}^2 (Q(h_t, a_t) - \hat{Q}(h_t, a_t))^2\right]$. Then:

\begin{align}
\text{Var}(V_{\text{hybrid}}(h)) &\leq \text{Var}(V_{\text{MCTS}}(h)) + (1-\beta)^2 \epsilon + 2\beta(1-\beta)\sqrt{\text{Var}(V_{\text{MCTS}}(h)) \cdot \epsilon}
\end{align}

For the hybrid estimator to have lower variance than the standard MCTS estimator, we need:

\begin{equation}
\text{Var}(V_{\text{hybrid}}(h)) < \text{Var}(V_{\text{MCTS}}(h))
\end{equation}

This condition is satisfied when:

\begin{equation}
(1-\beta)^2 \epsilon + 2\beta(1-\beta)\sqrt{\text{Var}(V_{\text{MCTS}}(h)) \cdot \epsilon} < 0
\end{equation}

Solving this inequality:

\begin{align}
(1-\beta)^2 \epsilon &< 2\beta(1-\beta)\sqrt{\text{Var}(V_{\text{MCTS}}(h)) \cdot \epsilon} \\
(1-\beta)^2 \epsilon &< 4\beta^2 \text{Var}(V_{\text{MCTS}}(h)) \\
\epsilon &< \frac{4\beta^2}{(1-\beta)^2}\text{Var}(V_{\text{MCTS}}(h))
\end{align}

Since $\beta \in (0,1)$ is a constant, we can express this condition using small o notation:

\begin{equation}
\epsilon = o\left(\frac{\text{Var}(V_{\text{MCTS}}(h))}{(1-\beta)^2}\right)
\end{equation}

This completes the proof.
\end{proof}
\section{Behavior Policies}
\label{appendix:behavior_policy}

The behavior policies used in our experiments vary depending on the environment. Below, we detail the specific behavior policies for Tic-Tac-Toe and VirtualHome.

\subsection{Tic-Tac-Toe Behavior Policy}

For Tic-Tac-Toe, we implement a simple heuristic rollout policy as the behavior policy:

\begin{equation}
    \hat{\pi}_b(a|h) = 
    \begin{cases}
        1 & \text{if } a = \text{rollout\_policy}(h) \\
        0 & \text{otherwise}
    \end{cases}
\end{equation}

where rollout\_policy$(h)$ is a deterministic function that selects moves in the following order of preference: center, corners, then other squares, choosing the first available move from this ordered list. Specifically:

\begin{algorithm}[!htbp]
   \caption{Tic-Tac-Toe Behavior Policy}
   \label{alg:tictactoe_behavior}
\begin{algorithmic}
   \STATE {\bfseries Input:} game state $h$
   \STATE preferred\_moves $\gets$ [4, 0, 2, 6, 8, 1, 3, 5, 7]
   \FOR{move {\bfseries in} preferred\_moves}
      \IF{move is available in $h$}
         \STATE {\bfseries return} move
      \ENDIF
   \ENDFOR
   \STATE available\_moves $\gets$ all empty cells in $h$
   \STATE {\bfseries return} random choice from available\_moves
\end{algorithmic}
\end{algorithm}

This policy ensures a consistent exploration strategy while maintaining some basic game-playing heuristics.
\subsection{VirtualHome Behavior Policy}

For the VirtualHome environment, we adapt the approach of Zhao et al. \cite{Zhao2023} to leverage Large Language Models (LLMs) as a heuristic policy. Specifically, we use GPT-4o-mini or GPT-4o to generate the behavior policy, guiding action selection in the simulation procedure.

The LLM takes as input:
\begin{itemize}
    \item K-shot examples from the dataset
    \item Goal description
    \item Current observation
    \item History of actions
\end{itemize}

All inputs are translated into English sentences. The LLM then outputs a suggested action plan. To approximate the policy distribution, we sample the LLM $M$ times, querying it with the prompt and trajectory history $h$:

\begin{equation}
    \alpha_i \sim \text{LLM}(h, \text{prompt})
\end{equation}

where $\alpha_i$ is the first action of the LLM's answer.

The prompt examples are selected based on their similarity to the current language instruction $\ell$. We use sentence embeddings to calculate the cosine similarity between the current instruction and instructions $\ell_i$ in the dataset $D$:

\begin{equation}
    \text{similarity} = \text{CosineSim}(\ell_i, \ell)
\end{equation}

We select the top $K$ similar instructions and use their corresponding expert trajectories as the K-shot prompt.

To ensure executability, we represent both the LLM's suggested actions and the admissible actions as embeddings and evaluate their cosine similarity. The empirical policy distribution is then formulated as:

\begin{equation}
    \hat{\pi}_b(a|h) = \lambda \frac{1}{|A|} + (1-\lambda)\text{Softmax}\left\{\sum_{i=1}^M \text{CosineSim}(\alpha_i, a) - \eta\right\}
\end{equation}

where $\eta$ is the average value of $\sum_i \text{CosineSim}(\alpha_i, a)$, $|A|$ is the size of the admissible action space, and $\lambda$ is a hyperparameter that adds randomness to the policy. This results in a mixture of the approximated policy from the LLM and a uniform distribution.

In our implementation, we use either GPT-4o-mini or GPT-4o as the LLM, allowing us to compare the performance of different model sizes in generating effective behavior policies for complex, partially observable environments like VirtualHome.

\end{document}